\documentclass[10pt,letterpaper]{article}
\usepackage[utf8]{inputenc}
\usepackage[T1]{fontenc}
\usepackage{times,microtype}
\usepackage[textheight=9in,textwidth=5.5in,top=1in,footskip=30pt]{geometry}
\usepackage[svgnames]{xcolor}
\usepackage{caption}
\usepackage{booktabs}
\usepackage{array}
\usepackage{float}
\usepackage{amsfonts}
\usepackage{graphicx}
\usepackage{amsmath,amssymb,mathtools,amsthm}
\usepackage[numbers,compress,sort]{natbib}
\usepackage[backref=page]{hyperref}
\hypersetup{
    pdftitle={Global Communication or Graph-Specific Memory?},
    pdfauthor={Hamed Shirzad and Danica J. Sutherland},
    pdfcreator={LaTeX},
    bookmarksnumbered,
    pdfdisplaydoctitle,
    colorlinks=true,
    breaklinks=true,
    citecolor=NavyBlue,
    linkcolor=NavyBlue,
    urlcolor=NavyBlue
}
\usepackage[capitalize,noabbrev]{cleveref}
\crefname{appendix}{Appendix}{Appendices}
\Crefname{appendix}{Appendix}{Appendices}

\newtheorem{proposition}{Proposition}

\theoremstyle{definition}

\theoremstyle{remark}

\newcommand{\R}{\mathbb{R}}

\title{Global Communication or Graph-Specific Memory?}

\author{Hamed Shirzad\\
University of British Columbia\\
\texttt{shirzad@cs.ubc.ca}
\and
Danica J. Sutherland\\
University of British Columbia\\
\texttt{dsuth@cs.ubc.ca}
}
\date{}

\begin{document}

\maketitle

\begin{abstract}
    Scalable Graph Transformers are commonly trained and evaluated on static large graphs in a transductive setup. Many scalable Graph Transformer components can be formulated as a constant-size shared memory, similar to virtual nodes, providing compressed information about the whole graph. The counterpart of these models in language models and other domains is justified as the input changes, and this mechanism learns to compress some useful information about the input. In transductive learning on a single fixed graph, however, any shared memory can be seen as a constant at test time.
    This raises the question of what exactly this shared memory does in this static setup.
    We give preliminary evidence that optimizing
    a shared memory directly performs similarly to global communication methods,
    and so normal local message-passing models can embed similar information in their weights.
    Thus, these settings may be a poor fit for evaluating global communication in graph neural networks.
\end{abstract}

\section{Introduction}\label{sec:intro}

Graph Transformers (GTs) are often motivated by their ability to conduct long-range communication across a large input graph,
thereby modeling non-local dependencies which are difficult for more local models to learn
\citep{dwivedi2022long,wu2022nodeformer,wu2023difformer,rampasek2022gps,shirzad2023exphormer,shirzad2024even}.
In practice, however,
these models are commonly evaluated by
full-batch transductive node classification
on a single graph
\citep{yang2016planetoid,hu2020ogb,platonov2023critical}.
As these long-range messages will only ever be passed on one graph,
it is natural to ask:
are they really necessary?

The clearest case is the use of virtual nodes in a Graph Neural Network (GNN)
\citep{gilmer2017neural,rampasek2022gps,shirzad2023exphormer},
where all nodes are connected to a global fixed-sized shared memory.
We can view this as consisting of two mechanisms:
all nodes collectively write into a shared memory,
constructing some compressed representation of (aspects of) the whole graph;
at the next layer, each node reads from this shared memory,
learning about the whole input.
As we will discuss later,
linear and kernelized linear attention models \citep[e.g.][]{katharopoulos2020linear,choromanski2021performer} variants of efficient Transformers
can be mathematically reformulated into an analogous writer/reader setup.

These efficient attention models were originally developed for sequential data trained on many varied inputs, such as processing different textual documents,
and thus different information is written and read when the model is applied to different documents;
we can think of this as a mapping $I \mapsto c_\theta(I)$.
In a transductive fixed-graph setting, as typically used for evaluating GTs, there is only one fixed graph input $I^*$.
For a given setting of model parameters $\theta$,
$c^* = c_\theta(I^*)$
is fixed:
the function $c_\theta$ itself matters only in what single vector $c^*$ it produces.
Thus, at inference time, we can think of $c^*$ simply as a parameter of the model.
This raises the questions:
\begin{itemize}
    \item Can an equally useful shared memory $c^*$ be learned directly, e.g.\ with gradient descent on a learnable matrix?
    \item Do the performance of these models in this static transductive setup imply anything about their ability to do long-range dependency modeling?
    \item What are the optimization and inductive bias implications?
    \item Can other GNNs learn to embed this cached memory into their weights, without an explicit shared memory mechanism?
    \item Do these shared writers add value when the input graph changes in small ways?
\end{itemize}

Prior empirical work has suggested that simple GNNs based on local message-passing
can outperform GT variants;
\citet{luo2024classic} specifically study this,
while the SGFormer results of \citet{liang2026citynetworks} also often show these models fall behind basic GNNs.
Our paper suggests a plausible explanation:
the evaluation framework fundamentally cannot tell what role, if any, the writer plays during learning.
Our class of shared-memory models does not cover all scalable GTs;
for instance, Exphormer \citep{shirzad2023exphormer}
is outside of it (as discussed in \cref{app:architectural-scope}),
and indeed it  achieves the best results in \cite{liang2026citynetworks}.

This paper is an early investigation toward the previous questions. We do provide some initial experimental and theoretical results in this work, but our focus is on clearly elaborating the potential problem to the community, and encouraging more caution in what static setups are capable of saying about learning long-range dependencies.

\begin{figure}[t]
\centering
\includegraphics[width=\linewidth]{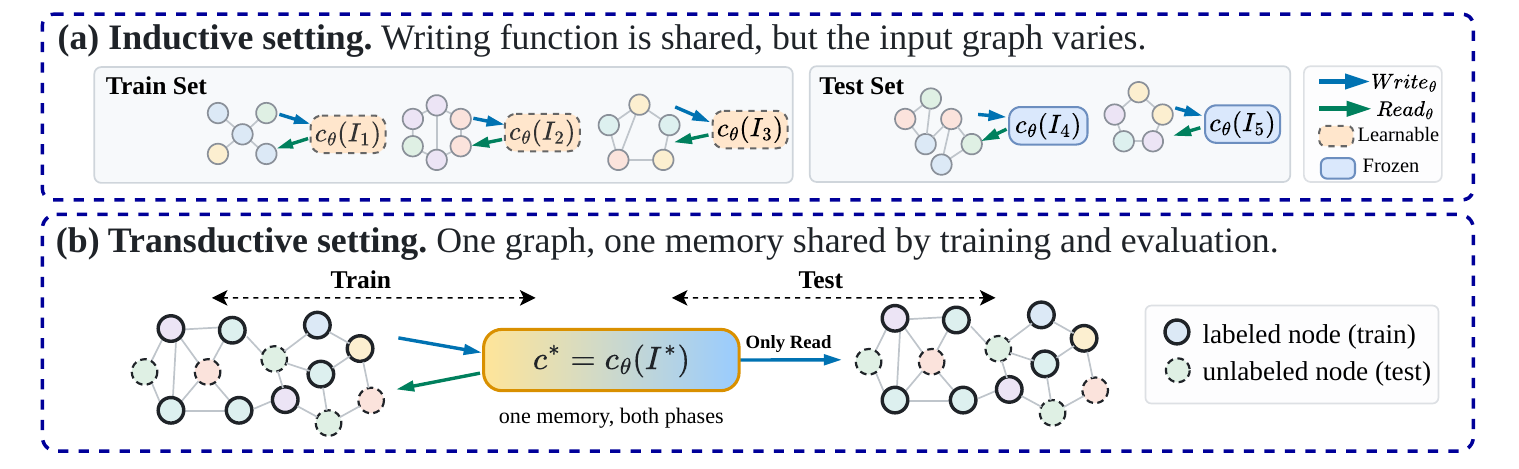}
\caption{\textbf{The role of a shared writer depends on the input domain.} (a) In the inductive setting, the writer is trained across multiple graphs and must produce an input-specific memory for each new test graph. (b) In the static transductive setting, training and evaluation use the same graph input ($I^*$). After training, its memory ($c^*=c_\theta(I^*)$) is fixed and can be recorded and reused at evaluation, while leaving the reader unchanged.}
\label{fig:identification}
\end{figure}

\section{From Virtual Nodes to Compact Shared Memory}\label{sec:setup}

\paragraph{Virtual nodes as explicit shared memory.}
Virtual nodes gather information from all nodes in the graph, and provide a shared memory.
All nodes can write information to this memory in the first layer; later layers can read from the shared memory and update it in each layer. Some GNN variants use one or more virtual nodes \citep{gilmer2017neural,shirzad2023exphormer}. Let their combined state at layer $\ell$ be $c^{(\ell)}(I)\in\R^M$,
where $M$ is the embedding size times the number of virtual nodes:
\[
\begin{aligned}
c^{(\ell)}(I)
&=\operatorname{Write}^{(\ell)}\!\left(
  \big(z_u^{(\ell-1)}(I)\big)_{u\in V},c^{(\ell-1)}(I)\right),\\
z_v^{(\ell)}(I)
&=\operatorname{Read}^{(\ell)}_v\!\left(
  \operatorname{Local}^{(\ell)}_v(I),c^{(\ell)}(I)\right),
\end{aligned}
\]

Where $V$ is the set of all nodes in the graph. The writer compresses the current graph state; the reader decides how each node uses it.

\paragraph{Kernelized linear attention as implicit shared memory.}
Kernelized attention implements a write-memory-read pattern without an explicit virtual node
\citep{katharopoulos2020linear,schlag2021linear,choromanski2021performer}.
Let $z_u\in\mathbb R^{d_h}$, let
$f_q,f_k:\mathbb R^{d_h}\to\mathbb R^{d_k}$ be query and key feature maps, and let
$f_v:\mathbb R^{d_h}\to\mathbb R^{d_v}$ be the value map.
Treating vectors as columns, the layer writes
\begin{equation}\label{eq:kernel-write}
S=\sum_u f_k(z_u)f_v(z_u)^\top
\in\mathbb R^{d_k\times d_v},
\qquad
\zeta=\sum_u f_k(z_u)
\in\mathbb R^{d_k},
\end{equation}
and node $v$ reads the same pair through its own query:
\begin{equation}\label{eq:kernel-read}
o_v=
\frac{S^\top f_q(z_v)}
     {\zeta^\top f_q(z_v)}
\in\mathbb R^{d_v}.
\end{equation}
In implementation, the denominator will be regularized or stabilized away from zero.

Different choices of the feature maps give different linear-attention variants. Taking $f_q$ and $f_k$ to be learned linear maps gives normalized softmax-free dot-product attention. If normalization is omitted, only $S$ is required. Positive or random feature maps give other kernelized variants, including linear-time approximations of softmax attention.

The pair $(S,\zeta)$ is a fixed-sized shared memory
$c=(S,\zeta) \in\mathbb R^{d_kd_v+d_k}$.
During training, this memory changes as the model parameters and node states change. After the parameters are fixed, each layer produces one fixed memory state on a given layer. For a multilayer model, we can see the combination of memories across all layers as the memory of the model:
$c = \bigl(c^{(1)},\ldots,c^{(L)}\bigr).$
If we remove the writing mechanism and make the reader mechanism directly use the stored $c$, it is easy to see that the model would produce identical outputs during inference time.

\section{What Constant-size Shared Memory Implies in a Transductive Setting}
\label{sec:collapse}

We denote the complete model input by $I=(G,X)$, where $G$ includes the graph and edge information, and $X$ includes the node features and any fixed positional or structural encodings. In common full-batch transductive node-classification benchmarks, the model is trained and evaluated using the same graph and features \citep{yang2016planetoid,kipf2017semi,hu2020ogb,platonov2023critical}. Only a subset of nodes is used during training, while other subsets are used for validation and test inference.

At inference time, we can remove the writing mechanism and make the reader use the stored shared memory $c$ at each layer to obtain exactly the same outputs. We call this \emph{exact replay}. While the mathematical argument is straightforward, we also test this by explicitly making these replacements across four models: Nodeformer (with fixed RF) \citep{wu2022nodeformer}, Difformer-s \citep{wu2023difformer}, SGFormer \citep{wu2023sgformer}, and Polynormer-r \citep{deng2024polynormer} across datasets Cora, Amazon-Photo, Pubmed, Actor \citep{sen2008collective,platonov2023critical,lim2021large,shchur2018pitfalls}, and observe all final logits are equal up to the computational stability threshold, as expected.

Another observation is that fixing $c^*$ after training at each layer also turns the reader at that layer into a shared \textit{nodewise} function
\begin{equation}
\label{eq:nodewise-read}
g^*(a) = \operatorname{Read}^*(a,c^*),
\end{equation}
where $a$ contains the retained local inputs of a node. Keeping this function gives an exact writer-free implementation, with a nodewise reader function at inference time.

\paragraph{Direct learnability of the shared memory.}
What if we eliminate the writing mechanism and learn $c$ directly using optimization algorithms alongside the other learnable parameters? Can we learn a useful shared memory through direct optimization? The number of learnable parameters in each layer according to \cref{eq:kernel-write} is $\Theta(d_h(d_v+d_k))$, while directly learning $c$ would require $\Theta(d_kd_v)$ parameters. In typical networks, $d_v$ (concatenated across the heads) and $d_h$ are the same or at least comparable, so the number of learnable parameters is comparable in the two setups. This raises the question of whether we can learn these parameters directly, while saving the writing computation in each forward pass. We test this experimentally across several architectures and datasets. The results provided in \cref{fig:learned-memory-bars} show that directly learning this shared memory is feasible, and gives comparable results to the original model. Experimental details are in \cref{sec:exp-learned-memory}.

\begin{figure*}[t]
    \centering
    \includegraphics[width=\textwidth]{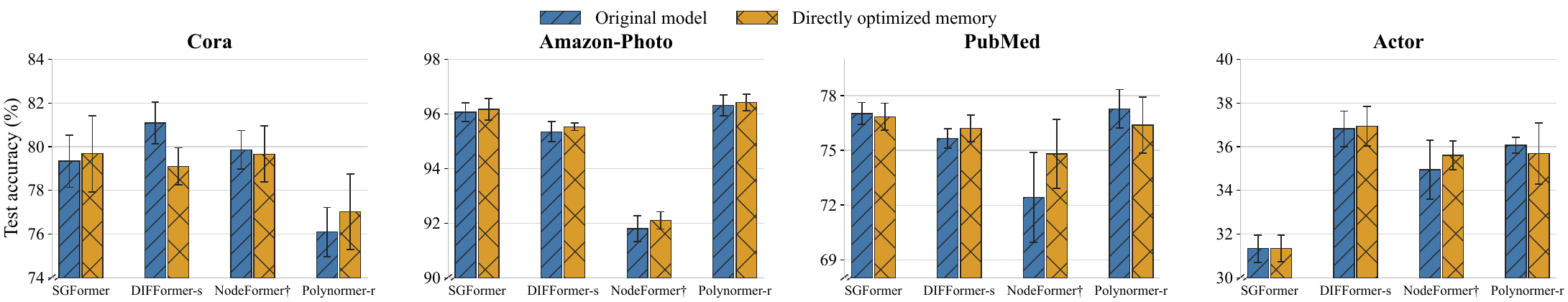}
    \caption{\textbf{Original models versus learned shared memories.}
    Bars show mean test accuracy over five matched seeds. Whiskers show marginal
    $95\%$ Student-$t$ confidence intervals. The learned-memory model has a
    higher mean in eleven of sixteen comparisons. Fourteen of the sixteen
    paired $95\%$ intervals for the memory-minus-original difference contain
    zero. The two exceptions have opposite signs: DIFFormer-s on Cora favors
    the original model by $2.0$ points, while NodeFormer on PubMed favors
    learned memory by $2.4$ points.}
    \label{fig:learned-memory-bars}
\end{figure*}

\paragraph{Hypothesis class of local GNNs and Kernelized Linear Graph Transformers.}
In \cref{eq:nodewise-read}, we saw that the read operation can be viewed as a nodewise operation after learning $c^*$. More particularly, \cref{eq:kernel-read} is a combination of $f_q$ and a linear map of $S$, normalized by a $\zeta$ dot product. Thus, from the perspective of representational strength, a sufficiently expressive nodewise operation at the end of each local message-passing layer can reproduce exactly the same output. This means that there exists a class of local GNNs whose hypothesis class contains that of these Graph Transformers. We formalize this mathematically in \cref{theory:hypothesis-class}. A sufficiently wide graph-specific Multi-Layer Perceptron (MLP), such as variants used in Graph Isomorphism Networks (GIN), can also interpolate this function at the finitely many node contexts encountered in each layer and reproduce the same predictions \citep{xu2018powerful}.

This result provides one possible explanation for the strong classical GNN results reported in earlier work. \citet{luo2024classic} show that carefully tuned GCN, GAT, and GraphSAGE models-with enough capacity-match or outperform recent Graph Transformer results on most of the node-classification datasets they study. \citet{liang2026citynetworks} also report classical GNN baselines that are close to or above SGFormer on their newly introduced long-range benchmarks. These comparisons do not show that the models learn the same function or use the same mechanism. Our result only shows why such predictive parity is possible on a fixed graph in a transductive setup. Whether gradient descent can reliably find the corresponding local GNN remains an open question.

\paragraph{Optimization perspective.}
The previous observations concern the representational capacity of the models; learnability is less clear. Can typical gradient-based optimization algorithms learn a useful graph-specific memory directly? Can a local message-passing GNN encode the same useful global information in its shared weights? The global writer may still affect optimization, regularization, and generalization even when its final memory can be replayed exactly. In \cref{fig:learned-memory-bars}, directly learned memories are often competitive with the original models. This shows that direct optimization can work in practice, but does not establish a general equivalence. Our initial theory in \cref{theory:optimization} explains why no unconditional equivalence is possible. It also gives exact recovery conditions for affine readers and shows that an affine writer acts as a preconditioner for learning the state. More general learnability guarantees for nonlinear multilayer models and ordinary GNN training remain open.

\paragraph{How generalizable is the writer against changes in the input?}
While the transductive setup commonly used in this area does not include changes to the input $I$ between training, validation, and test phases, one can view this as a simplification of real-world tasks. In settings such as social networks, for example, the graph may change over time, while we still expect the trained model to work under small changes to the input. In this case, we have an input graph $I_{\mathrm{train}}$ observed during training, but during the test phase, small changes may occur either in the graph structure $G$ or in the node features $X$. The writer has been trained on only one instance, but it may still generalize to sufficiently close variants. While this is possible, it is not clear if the model will learn a generalizable memory there from a single graph. We did a preliminary experiment described in \cref{sec:exp-perturbations}. Across Cora and Actor, recomputing the writer gives little consistent advantage over replay under feature replacement or edge rewiring. Training on perturbed inputs helps some configurations.
Future work should explore when the writer mechanism can generalize, and when it does not.

\Cref{fig:input-perturbations} shows these results; the full protocol is given in \cref{sec:exp-perturbations}.

\begin{figure*}[t]
    \centering
    \includegraphics[width=\textwidth]{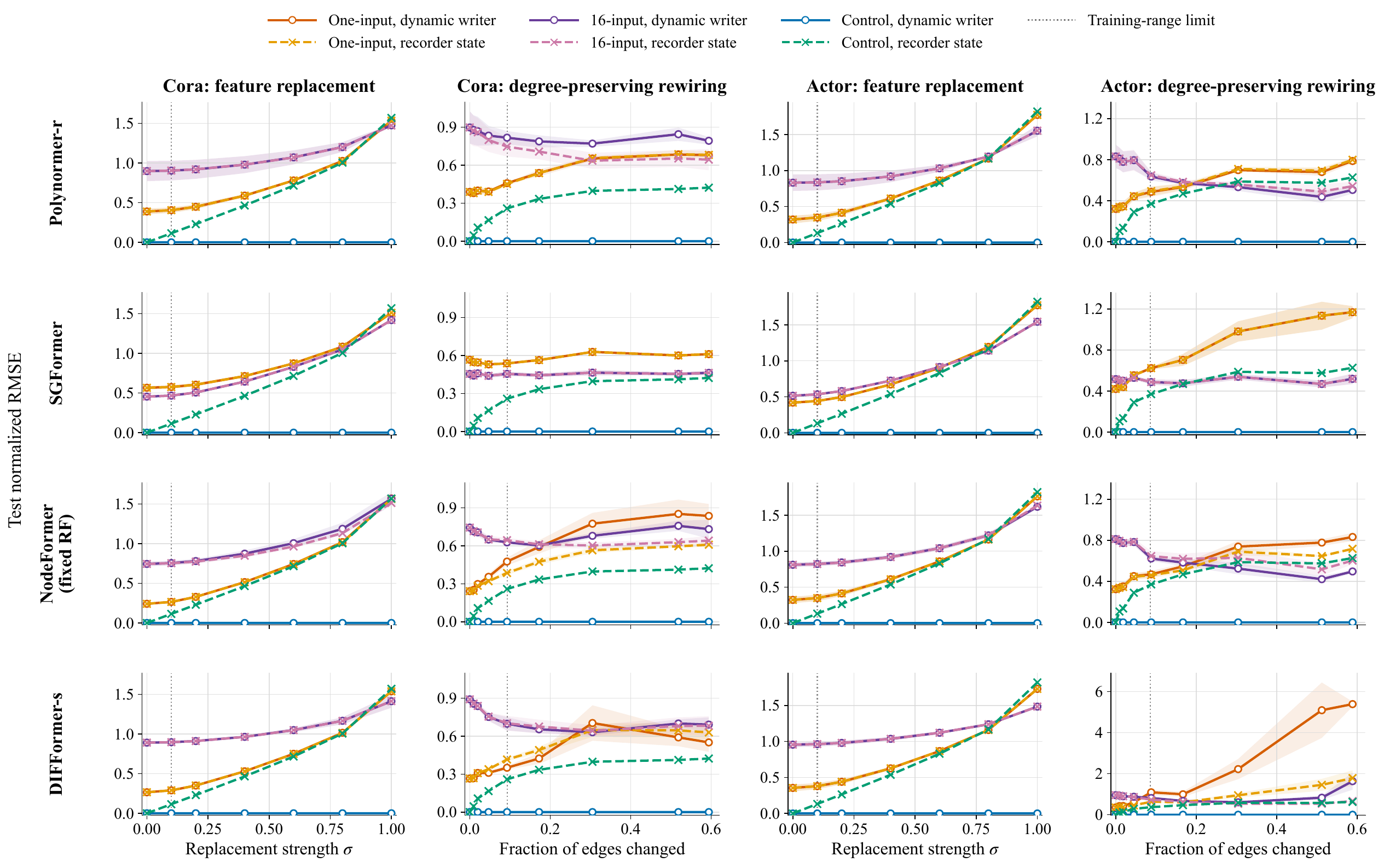}
    \caption{\textbf{Response to feature and topology changes.}
    Solid lines recompute the current writer state. Dashed lines reuse the clean recorder state. Orange curves use one clean training input, purple curves use one checkpoint trained on \(16\) nearby inputs, and the blue and green curves show the aligned writer-only control. Shaded regions show variation across three optimization seeds. The dotted line marks the largest marginal perturbation used during multi-input training.}
    \label{fig:input-perturbations}
\end{figure*}

\paragraph{Experiments on fixed long-range tasks.}
\Cref{fig:static-long-range} compares the original models, directly optimized memory, and a two-hop GNN on a fixed target with known long-range dependence; details are given in \cref{app:static-long-range}. \Cref{fig:city-memory} extends the original-versus-memory comparison to the Paris, Shanghai, and Los Angeles City-Networks datasets \citep{liang2026citynetworks}; details are given in \cref{app:city-memory}.

\begin{figure*}[t]
    \centering
    \includegraphics[width=\textwidth]{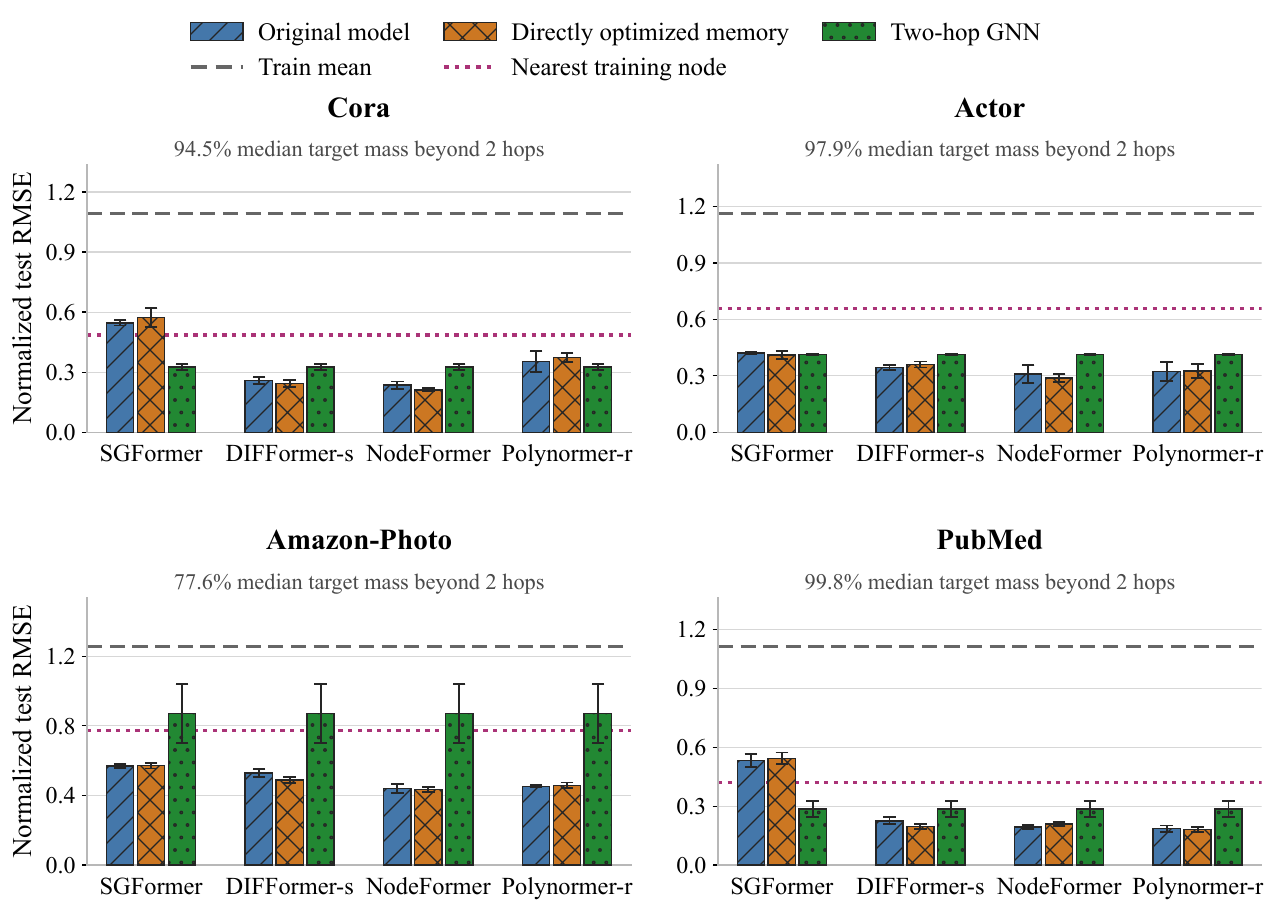}
    \caption{\textbf{Learning one fixed long-range target.}
    Lower normalized test RMSE is better. Bars compare the original model,
    directly optimized shared memory, and a two-hop GNN. The same GNN result is
    repeated within each model group. Horizontal lines show the train mean and
    nearest-training-node baselines. Whiskers are $95\%$ confidence intervals
    over five optimization seeds.}
    \label{fig:static-long-range}
\end{figure*}

\begin{figure*}[t]
    \centering
    \includegraphics[width=\textwidth]{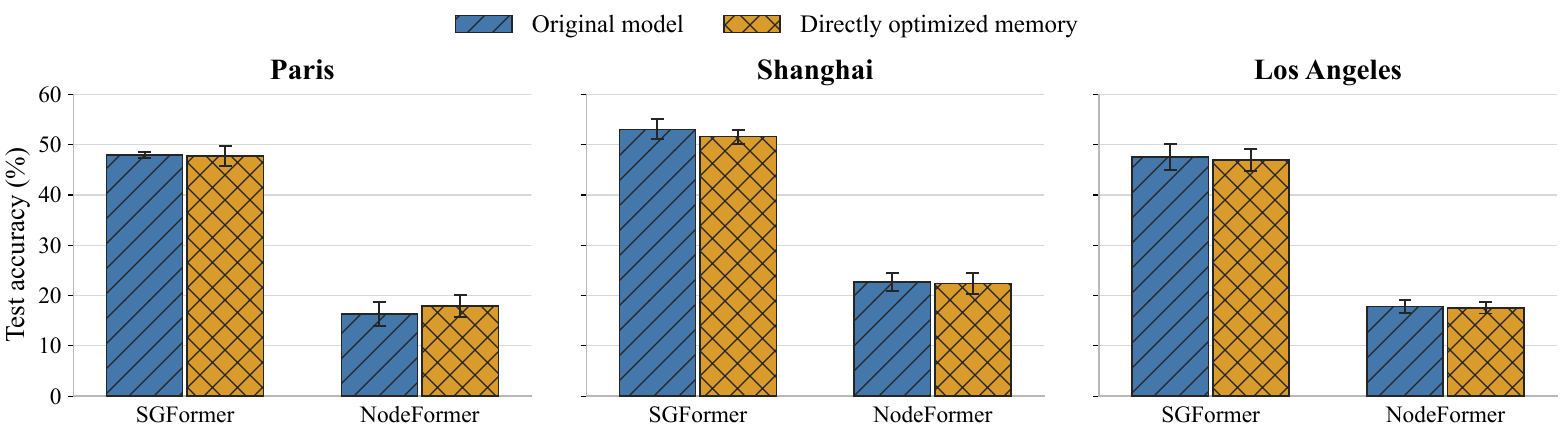}
    \caption{\textbf{Directly optimizing shared memory on City-Networks.}
    Bars show mean test accuracy for the original model and its directly
    optimized shared-memory counterpart. Whiskers are marginal $95\%$
    Student-$t$ confidence intervals.}
    \label{fig:city-memory}
\end{figure*}

\section{Conclusion \& Future Work}

Many variants of efficient Graph Transformers, similar to virtual nodes, can be seen as a shared memory mechanism.
While developed in settings where the shared memory learns a mechanism to extract useful information for the task from each different input, its role in the standard transductive setup, with one fixed graph as the input, is less clear.

Across several models and datasets, directly optimizing the shared memory performs similarly to the original writer. We also show that, on one fixed input, the covered compact-memory models have no larger prediction set than sufficiently wide graph-specific local GNNs. This does not mean that the full model families are equivalent or that the writer is unimportant during training. In our perturbation experiments, however, recomputing the memory gives no consistent advantage over replay. These results suggest that static transductive accuracy alone cannot establish that input-conditioned global communication was necessary.

\clearpage
\bibliographystyle{unsrtnat}
\bibliography{refs}

\clearpage
\appendix
\raggedbottom
\crefalias{section}{appendix}
\crefalias{subsection}{appendix}

\setlength{\textfloatsep}{8pt plus 2pt minus 2pt}
\setlength{\floatsep}{8pt plus 2pt minus 2pt}
\setlength{\intextsep}{8pt plus 2pt minus 2pt}
\captionsetup{skip=4pt}

\section{Related Work}\label{sec:related}

\paragraph{Shared memory and scalable global communication.}
Linear attention is commonly written as recurrent or fast-weight memory: keys and values update shared moments that later queries read \citep{katharopoulos2020linear,schlag2021linear,choromanski2021performer}. This algebraic interpretation is established; our contribution is to ask what it lets one identify when evaluation contains only one fixed graph input. Graph Transformers use global or structurally biased attention, often alongside local message passing, and scalable designs use kernelized attention, diffusion, polynomial aggregation, sparse routing, or bounded auxiliary tokens \citep{ying2021graphormer,kreuzer2021san,rampasek2022gps,muller2024attending,wu2022nodeformer,wu2023difformer,wu2023sgformer,deng2024polynormer,shirzad2023exphormer,shirzad2024even}.

\paragraph{Attention and baseline reassessments.}
Prior work has replaced input-dependent attention by averages or simpler propagation modules \citep{hassid2022attend,zhou2024rethinking}. Carefully tuned GNN baselines match or exceed recent Graph Transformer results on many transductive node-classification datasets \citep{luo2024classic}; separately, tuning and evaluation corrections substantially reduce the reported GNN-GraphGPS gap on the Long-Range Graph Benchmark \citep{tonshoff2023where}. Such outcome-level parity motivates the mechanism question, but does not establish functional, optimization, or hypothesis-class equivalence.

\paragraph{Tests of long-range dependence.}
The Long-Range Graph Benchmark introduced tasks intended to expose the use of distant information \citep{dwivedi2022long}. Later work measures interaction range through graph distance and model influence \citep{bamberger2025measuring}, combines a long-range transductive task with Jacobian-based influence measurements \citep{liang2026citynetworks}, constructs targets with provable long-range dependence \citep{mathys2026lrim}, or tests inductive generalization across generated graph families and controlled shifts \citep{vanlangendonck2026graphuniverse}. These studies ask whether distant information matters, whether a model responds to it, or whether behavior generalizes across inputs. Our work talks about a specific transductive setup and whether the shared memory captures long-range dependencies and whether it is necessary or advantageous over local message-passing models in this particular setup.

\paragraph{Expressivity and compressibility.}
Expressivity results compare architectural classes \citep{cai2023connection,loukas2020what}; we instead ask what fixed-input evaluation identifies about one trained computation.  \citet{shirzad2024compressibility} derive conditions for reducing a Graph Transformer's hidden width while approximately preserving its outputs, with corresponding attention-score guarantees under their assumptions.

\section{Theoretical Support}\label{sec:theory}

Exact replay shows that the writer is unnecessary after training on one fixed input. We next ask two harder questions. Can the effect of the shared memory be absorbed into an ordinary GNN? Can gradient descent learn such a fixed state directly? We give an exact answer to the first question and initial results for the second.

\subsection{Hypothesis Class}\label{theory:hypothesis-class}
\paragraph{Compiling the memory into local weights.}
Consider a layer of the form
\begin{equation}
    h_v^{\ell+1}
    =
    R_\ell\left(a_v^\ell,c^\ell\right),
    \label{eq:memory-reader}
\end{equation}
where \(a_v^\ell\) contains the node state and all retained local messages. The state \(c^\ell\) is shared by every node. We assume that every nonlocal pathway covered by the claim passes through \(c^\ell\).

On a fixed input \(I^*\), the writer produces one fixed state \(c^{\ell,*}\). The layer therefore becomes the nodewise function
\begin{equation}
    g_\ell(a)
    =
    R_\ell\left(a,c^{\ell,*}\right).
    \label{eq:compiled-reader}
\end{equation}

\begin{proposition}[Fixed-input compilation]
For any trained compact-memory model on \(I^*\), there exists a sufficiently wide writer-free message-passing GNN that uses the same local edges and produces exactly the same predictions on \(I^*\).
\label{prop:fixed-compilation}
\end{proposition}

\begin{proof}[Proof sketch]
At each layer, the fixed graph produces finitely many local contexts \(a_v^\ell\). For each context, the required output is
\(R_\ell(a_v^\ell,c^{\ell,*})\).
Equal contexts have equal required outputs because the reader and shared state are the same.

A sufficiently wide ReLU MLP can interpolate these finitely many input-output pairs exactly \citep{yun2019small}. We replace the memory-conditioned reader by this shared MLP. Repeating this construction layer by layer preserves every node state and final prediction.
\end{proof}

Let \(\mathcal P_{\rm mem}(I^*)\) denote the predictions attainable by compact-memory models on \(I^*\). Let \(\mathcal P_{\rm GNN}(I^*)\) denote those attainable by sufficiently wide local GNNs. Then
\begin{equation}
    \mathcal P_{\rm mem}(I^*)
    \subseteq
    \mathcal P_{\rm GNN}(I^*).
    \label{eq:prediction-inclusion}
\end{equation}
If the compact-memory family contains the corresponding local GNN when its writer is disabled, the two prediction sets are equal.

This equality holds only on \(I^*\). It does not imply that the two families represent the same functions over changing graphs or features. The compiled GNN can also require graph-specific weights and large width. This is therefore a representational result, not a learnability or efficiency result.

\subsection{Optimization Perspective}
\label{theory:optimization}

\paragraph{Why general optimization equivalence is impossible.}
Exact replay does not imply that removing the writer preserves training.
Consider a state $c=(c_1,c_2)$. A labeled node reads $c_1$, while a held-out
node reads $c_2$. Both desired outputs are $1$.

Suppose the writer enforces
\[
    c=(\psi,\psi).
\]
Training on the labeled node moves $\psi$ toward $1$. Both entries then
approach $1$, so the held-out prediction is also correct.

Now optimize $c_1$ and $c_2$ independently from zero. The labeled loss updates
$c_1$, but gives no gradient to $c_2$. Training approaches $c=(1,0)$. It fits
the labeled node but fails on the held-out node.

The writer can therefore couple state directions that would otherwise be
learned independently. Its parameterization can change which solution is
reached, as is common in underdetermined models
\citep{gunasekar2017implicit,gunasekar2018geometry}. No general theorem can
claim that a writer and a free state always learn the same model.

\paragraph{Learning an affine shared state.}
Suppose predictions are affine in the state:
\begin{equation}
    p_L(c)=A_Lc+b_L,
    \qquad
    p_U(c)=A_Uc+b_U,
    \label{eq:affine-state-readers}
\end{equation}
where $L$ and $U$ are the labeled and held-out nodes. Let
\[
    \mathcal L_L(c)
    =
    \frac12\lVert A_Lc+b_L-y_L\rVert^2,
\]
and let $c^*$ be a replay state that minimizes this loss.

\begin{proposition}[Gradient descent on a fixed state]
Optimize $c$ by gradient descent on $\mathcal L_L$. If
\[
    0<\eta<\frac{2}{\lVert A_L\rVert_2^2},
\]
the iterates converge to a global minimizer. Their limit is
\begin{equation}
    c_\infty
    =
    c^*
    +
    P_{\ker(A_L)}(c_0-c^*),
    \label{eq:affine-state-limit}
\end{equation}
where $P_{\ker(A_L)}$ projects onto the state directions that the labeled
nodes cannot observe.

The replay state is recovered from every initialization when $A_L$ has full
column rank. More generally, its held-out predictions are recovered from every
initialization exactly when
\begin{equation}
    \ker(A_L)
    \subseteq
    \ker(A_U).
    \label{eq:state-observability}
\end{equation}
\label{prop:affine-state}
\end{proposition}

\begin{proof}
Because $c^*$ minimizes the labeled loss, it satisfies the normal equation
\[
    A_L^\top(A_Lc^*+b_L-y_L)=0.
\]
One gradient step therefore gives
\[
    c_{t+1}-c^*
    =
    (I-\eta A_L^\top A_L)(c_t-c^*).
\]
Under the stated step size, every component in
$\operatorname{range}(A_L^\top)$ converges to zero, while every component in
$\ker(A_L)$ remains unchanged. This proves
\cref{eq:affine-state-limit}. The state is recovered for every initialization
exactly when $\ker(A_L)=\{0\}$. Held-out predictions are recovered for every
initialization exactly when $A_Ud=0$ for every $d\in\ker(A_L)$, which is
\cref{eq:state-observability}.
\end{proof}

Condition~\eqref{eq:state-observability} has a simple meaning. Any state
direction invisible to the labeled nodes must also be irrelevant to the
held-out nodes. The labels do not need to identify the whole state. They only
need to identify the parts used at evaluation.

Equation~\eqref{eq:affine-state-limit} also shows the role of initialization.
Gradient descent selects the labeled-loss minimizer closest to $c_0$. In
particular, zero initialization selects the minimum-norm minimizer. This is a
simple form of implicit bias \citep{gunasekar2018geometry}.

A similar choice can be made explicitly. Consider
\begin{equation}
    \mathcal L_\lambda(c)
    =
    \mathcal L_L(c)
    +
    \frac{\lambda}{2}\lVert c-c_{\rm ref}\rVert^2,
    \qquad \lambda>0.
    \label{eq:regularized-state}
\end{equation}
This objective has a unique minimizer. As $\lambda$ approaches zero, it
selects the labeled-loss minimizer closest to $c_{\rm ref}$. Initialization
and regularization therefore control state directions that the labels do not
determine. They cannot guarantee correct held-out behavior when
\cref{eq:state-observability} fails.

For one-layer unnormalized linear attention with fixed queries,
\[
    O=QS,
\]
the shared memory $S$ is a shared linear weight matrix.
Proposition~\ref{prop:affine-state} applies directly when $S$ is optimized with
squared loss. For normalized kernel attention, it also applies to $S$ when the
normalization state is fixed. Jointly learning all states in a nonlinear
multilayer model remains outside this result.

\paragraph{When writer training equals direct-state training.}
Suppose the writer is affine in its parameters:
\begin{equation}
    c=B\psi+d,
    \label{eq:affine-writer}
\end{equation}
where $B$ and $d$ are fixed. One gradient step on $\psi$ gives
\begin{equation}
    c_{t+1}
    =
    c_t
    -
    \eta BB^\top\nabla_c\mathcal L(c_t).
    \label{eq:writer-preconditioner}
\end{equation}
Thus writer training is exactly preconditioned gradient descent on the shared
state. If $BB^\top=\gamma I$, ordinary direct-state training follows the same
trajectory after multiplying its learning rate by $\gamma$. This requires
matching initialization and no writer-specific regularization.

If $B$ is not full row rank, the writer restricts which state directions can
be reached. If $B$ is full row rank, it can still change the speed and geometry
of optimization. An affine writer is therefore a parameterization and
preconditioner for learning the fixed state.

\paragraph{An exact extension to nonlinear multilayer models.}
The same parameterization view extends to nonlinear networks of arbitrary
depth. The readers and local layers can also be trained jointly. We use a
smooth-parameterization argument \citep{levin2025effect}.

Fix the complete input $I^*$ and consider a deterministic model with $D$
layers. Write its parameters as $w=(\theta,\psi)$, where $\theta$ contains
all retained parameters and $\psi$ contains the removed writer-only
parameters. Let
\[
    C(w)=\bigl(c^1(w),\ldots,c^D(w)\bigr)
\]
stack the memories produced by the complete original forward pass.
Thus $C$ includes all dependencies through earlier layers and through
$\theta$.

Let $\mathcal L(\theta,c)$ be the prediction loss of the writer-free model,
with all layer memories supplied as independent variables. Let
$\mathcal J(w)$ be the original model's prediction loss. Both use the same
loss, without additional writer-specific penalties. Define
\begin{equation}
    \Phi(w)=(\theta,C(w)).
    \label{eq:nonlinear-writer-map}
\end{equation}
Here, $D\Phi(w)$ denotes the Jacobian of $\Phi$.

\begin{proposition}[Nonlinear multilayer writers as parameterizations]
\label{prop:nonlinear-writer}
Assume that $\Phi$ and $\mathcal L$ are continuously differentiable on open
neighborhoods of the points considered. Then
\[
    \mathcal J(w)=\mathcal L(\Phi(w)).
\]
Along any gradient-flow trajectory
$\dot w(t)=-\nabla\mathcal J(w(t))$ that stays in these neighborhoods,
$z(t)=\Phi(w(t))$ satisfies
\begin{equation}
    \dot z(t)
    =
    -D\Phi(w(t))D\Phi(w(t))^\top
    \nabla\mathcal L(z(t)).
    \label{eq:nonlinear-writer-flow}
\end{equation}

At any $w^*$ where $D\Phi(w^*)$ has full row rank, with
$z^*=\Phi(w^*)$,
\begin{equation}
    \nabla\mathcal J(w^*)=0
    \quad\Longleftrightarrow\quad
    \nabla\mathcal L(z^*)=0.
    \label{eq:nonlinear-writer-stationarity}
\end{equation}
Under the same rank condition, $w^*$ is a local minimum of $\mathcal J$
if and only if $z^*$ is a local minimum of $\mathcal L$.
\end{proposition}

\begin{proof}
For each parameter setting $w$, supplying $C(w)$ to the writer-free model
reproduces the original computation. This follows by induction over layers.
At each layer, the preceding node states, retained parameters, and supplied
memory agree. The outputs therefore agree as well. Hence
$\mathcal J=\mathcal L\circ\Phi$.

The chain rule gives
\[
    \nabla\mathcal J(w)
    =
    D\Phi(w)^\top\nabla\mathcal L(\Phi(w)).
\]
Along gradient flow,
\[
    \dot z
    =
    D\Phi(w)\dot w
    =
    -D\Phi(w)D\Phi(w)^\top\nabla\mathcal L(z),
\]
which proves \cref{eq:nonlinear-writer-flow}.

If $D\Phi(w^*)$ has full row rank, its transpose has zero kernel. The
gradient identity therefore proves
\cref{eq:nonlinear-writer-stationarity}.

If $z^*$ is a local minimum of $\mathcal L$, continuity of $\Phi$ makes
$w^*$ a local minimum of $\mathcal J$. Conversely, full row rank and the
inverse function theorem give a continuous local map $s$ satisfying
\[
    s(z^*)=w^*,
    \qquad
    \Phi(s(z))=z.
\]
This map can be obtained by fixing the parameter coordinates outside an
invertible square minor of $D\Phi(w^*)$. If $w^*$ is a local minimum,
then for every sufficiently nearby $z$,
\[
    \mathcal L(z)
    =
    \mathcal J(s(z))
    \ge
    \mathcal J(w^*)
    =
    \mathcal L(z^*).
\]
Thus $z^*$ is a local minimum of $\mathcal L$.
\end{proof}

The rank condition has a direct interpretation. Since
\[
    D\Phi(w)
    =
    \begin{pmatrix}
        I & 0\\
        D_\theta C(w) & D_\psi C(w)
    \end{pmatrix},
\]
it has full row rank exactly when $D_\psi C(w)$ has full row rank.
The writer parameters must therefore locally control every memory
direction while the retained parameters are held fixed. We have not
established this condition for the experimental architectures.

Equation~\eqref{eq:nonlinear-writer-flow} holds without the rank condition
and does not linearize the network. The writer induces a
positive-semidefinite preconditioner on the retained parameters and all
layer memories. This preconditioner depends on $w(t)$ and need not be
determined by $z(t)$ alone. The result does not imply identical finite-step
gradient-descent updates or convergence to the same solution from
independent initializations.

\paragraph{A local extension beyond affine readers.}
For a nonlinear reader, its behavior near a state $c_0$ is described by its
Jacobian. In a small neighborhood, $A_L$ and $A_U$ in
Proposition~\ref{prop:affine-state} can be replaced by the labeled and held-out
Jacobians. The condition
\[
    \ker(J_L)\subseteq\ker(J_U)
\]
then gives a local observability test. Every locally invisible direction for
the labeled nodes must also be locally irrelevant to the held-out nodes.

If the complete state objective is smooth and satisfies a
Polyak--Lojasiewicz condition, gradient descent converges globally at a linear
rate \citep{karimi2016linear}. This condition can hold without convexity.
However, it has not been established for the nonlinear Graph Transformers
considered here. It is a possible route for future analysis, not a claim about
their current training dynamics.

\subsection{Global Coupling Through Shared GNN Weights}
A local GNN has another graph-wide channel during training. Let $f_v(\theta)$
be a scalar prediction at node $v$. After one gradient step on the labeled
nodes,
\begin{equation}
\begin{split}
    f_v(\theta^+)-f_v(\theta)
    =
    -\eta\sum_{u\in L}
    \left\langle
        \nabla_\theta f_v,
        \nabla_\theta f_u
    \right\rangle
    \ell'_u
    +O(\eta^2).
\end{split}
\label{eq:cross-node-training}
\end{equation}
The inner product measures how an update caused by labeled node $u$ changes
node $v$. It can be nonzero even when the nodes are outside each other's
forward receptive fields because they use the same parameters. This is the
finite-width, first-order tangent-kernel view of neural-network training
\citep{jacot2018neural}; related graph neural tangent kernels have been
studied for GNNs \citep{du2019graph}.

The first-order expression becomes exact for the linearized model
\[
    f(\theta)
    =
    f(\theta_0)+J(\theta-\theta_0),
\]
whose Jacobian $J$ and kernel $K=JJ^\top$ are fixed. The same dynamics describe
the corresponding infinite-width limit when its neural tangent kernel remains
constant. The next result characterizes what is learned in this regime.

\begin{proposition}[Learning across nodes with a fixed tangent kernel]
\label{prop:fixed-kernel-gnn}
Let $f(t)\in\mathbb R^n$ contain one prediction per node of a linearized GNN on
a fixed graph. Let $J$ be its fixed Jacobian and $K=JJ^\top$. Train on labeled
nodes $L$ with the squared loss
\[
    \mathcal L(f_L)=\frac12\lVert f_L-y_L\rVert^2.
\]
Assume that $K_{LL}$ is positive definite \citep[see conditions in][]{carvalho2024positivityneuraltangentkernel}. Under gradient flow,
\begin{equation}
    f_L(t)-y_L
    =
    \exp(-K_{LL}t)\bigl(f_L(0)-y_L\bigr),
    \label{eq:fixed-kernel-labeled}
\end{equation}
and
\begin{equation}
\begin{split}
    f_U(t)
    =f_U(0)
    +K_{UL}K_{LL}^{-1}
    \left(I-\exp(-K_{LL}t)\right)
    \left(y_L-f_L(0)\right).
\end{split}
\label{eq:fixed-kernel-path}
\end{equation}
Consequently,
\begin{equation}
    f_U(\infty)
    =
    f_U(0)
    +K_{UL}K_{LL}^{-1}\left(y_L-f_L(0)\right).
    \label{eq:fixed-kernel-limit}
\end{equation}
The same limit is reached by gradient descent on the linearized model when
\[
    0<\eta<\frac{2}{\lambda_{\max}(K_{LL})}.
\]
For a desired held-out target $y_U$, exact recovery occurs if and only if
\begin{equation}
    y_U-f_U(0)
    =
    K_{UL}K_{LL}^{-1}\left(y_L-f_L(0)\right).
    \label{eq:fixed-kernel-recovery}
\end{equation}
\end{proposition}

\begin{proof}
With a constant tangent kernel, gradient flow on the labeled squared loss
satisfies
\[
    \dot f_L(t)=-K_{LL}(f_L(t)-y_L),
    \qquad
    \dot f_U(t)=-K_{UL}(f_L(t)-y_L).
\]
The first linear differential equation gives
\cref{eq:fixed-kernel-labeled}. Substituting this solution into the second
equation and integrating gives \cref{eq:fixed-kernel-path}. Positive
definiteness of $K_{LL}$ makes the exponential term vanish as $t$ approaches
infinity, which proves \cref{eq:fixed-kernel-limit}.

For gradient descent, the labeled residual obeys
\[
    f_L^{(t+1)}-y_L
    =
    (I-\eta K_{LL})(f_L^{(t)}-y_L).
\]
The stated step size makes every eigenvalue of $I-\eta K_{LL}$ have magnitude
less than one. Summing the corresponding geometric series in the held-out
updates gives the same limit. Finally, comparing that limit with $y_U$ proves
\cref{eq:fixed-kernel-recovery}.
\end{proof}

This proposition gives an exact class of graph-wide signals that a local GNN
can learn through shared weights. If the complete target satisfies
\[
    y-f(0)=K_{:,L}\alpha
\]
for some vector $\alpha$, then
$y_L-f_L(0)=K_{LL}\alpha$ and
\cref{eq:fixed-kernel-limit} gives $f_U(\infty)=y_U$. The distance between a
labeled node and a held-out node does not appear in this condition. The signal
can therefore vary across the whole graph even when each forward pass is
local.

For the present paper, $y$ may be the predictions produced by a fixed compact
memory or the target of a static transductive task. Proposition~\ref{prop:fixed-kernel-gnn}
shows that a local GNN learns this graph-wide signal exactly when its residual
lies in the kernel interpolation space determined by the labeled nodes. It
does not claim that every compact-memory target satisfies this condition.

Equation~\eqref{eq:cross-node-training} alone only shows that distant nodes can
be coupled. Proposition~\ref{prop:fixed-kernel-gnn} goes further: in the
fixed-kernel regime, it gives the complete training trajectory, the held-out
limit, and an exact recovery condition. For ordinary finite-width nonlinear
training, the kernel changes and this proposition need not apply.

Overall, compact-memory effects can be represented exactly by a sufficiently
expressive local GNN on one fixed input. Direct optimization recovers them in
the affine case when labeled nodes observe every evaluation-relevant
direction. In a fixed-kernel regime, shared GNN weights recover exactly the
graph-wide signals in the corresponding labeled-node kernel interpolation
space. Outside these regimes, writers can change optimization, coupling, and
implicit bias, so training equivalence is not guaranteed.

\section{Architectural Scope}\label{app:architectural-scope}

Our claim concerns components, not complete architectures. A nonlocal component is covered when all graph-wide information passes through a bounded state shared by every node. The size of this state must not grow with the number of nodes or edges. Preprocessing, model parameters, and evaluation-time randomness must also be fixed. On one fixed input, the resulting state can then be recorded and replayed exactly.

\paragraph{Directly covered models.}
Kernelized and linear attention provide the clearest examples. Their global computation factors through a fixed number of key-value moments and normalization terms \citep{katharopoulos2020linear,schlag2021linear,choromanski2021performer}. This includes the relevant components of SGFormer, the deterministic kernel path of NodeFormer, DIFFormer-s, Polynormer-r, and GraphGPS with Performer attention \citep{wu2023sgformer,wu2022nodeformer,wu2023difformer,deng2024polynormer,rampasek2022gps}. Replay must store every input-dependent moment and normalizer used by the reader.

\paragraph{Explicit shared tokens.}
A fixed number of virtual or global nodes also forms a bounded shared memory. Nodes write to these tokens and later read from them. The virtual-node component of an architecture is therefore covered when the number and width of the tokens do not grow with the graph \citep{gilmer2017neural,rampasek2022gps,shirzad2023exphormer}. If the model also has local, sparse, or pairwise routes, those routes remain unchanged. The claim applies only to the shared-token component.

\paragraph{Models covered under conditions.}
Random-feature models are covered only after their evaluation-time randomness is fixed. For NodeFormer, our audit therefore concerns its deterministic kernel path, not every stochastic or relational component. Anchor, codebook, cluster, and latent-token methods fit the same formulation only when the complete shared bank has graph-independent size. If the number of stored items grows with the graph, replay is no longer a compact shared state. A low-rank score matrix alone is also insufficient. The implemented computation must actually factor through a bounded write--read state.

\paragraph{Models outside the direct claim.}
Full attention and graph-scaling sparse attention use node-pair or edge-indexed routes. These routes are not summarized by one bounded state. This includes the sparse edge pathways of Exphormer and Spexphormer, GraphGPS with BigBird attention, and the non-factorized DIFFormer-a variant \citep{shirzad2023exphormer,shirzad2024even,rampasek2022gps,zaheer2020bigbird,wu2023difformer}. Replaying them would generally require information whose size grows with the graph. Efficiency by itself does not imply compact shared memory.

\paragraph{Mamba and graph state-space models.}
Mamba maintains a bounded recurrent state and therefore belongs to the broader write-state-read view \citep{gu2024mamba}. However, the state changes along the scan. Different positions can read different states, so replaying the full sequence may require a graph-sized state trajectory. Graph Mamba models also differ in how they tokenize, order, and scan graph nodes \citep{behrouz2024graphmamba,wang2024graphmamba}. A variant is directly covered only if its graph-wide pathway ends in a bounded state shared by all relevant node readers. Variants that retain a different scan state for each node are not directly covered. A model-by-model analysis is therefore needed.

\paragraph{Interpretation.}
For hybrid models, replay removes only the covered writer. Every other pathway stays active. Exact replay also does not say that a factorized model is equivalent to the dense attention rule it approximates. This scope concerns what a fixed-input score can identify about a mechanism. It does not judge the predictive quality, expressivity, or scalability of the cited models.

\section{Experiments}
\label{sec:experiments}

We use the following experiments to separate different questions. Exact replay checks that the implementations have the shared-memory form described in \cref{sec:setup}. The static experiments ask whether the shared memory can instead be learned as graph-specific parameters. The perturbation experiment asks whether a writer trained on fixed inputs responds usefully when the input changes.

\subsection{Exact Replay}
\label{sec:exp-replay}

We first audit four models, Nodeformer, Difformer-s, SGFormer, and Polynormer-r, on Cora, Amazon-Photo, PubMed, and Actor datasets \citep{sen2008collective,platonov2023critical,lim2021large}. We train each model on each of the datasets once. We then run its validation-selected checkpoint twice on the same input. The first run records the graph-wide state at every layer. The second run removes the corresponding writes and uses the stored shared memory with the learned read mechanism. All other parameters remain unchanged.

We construct the recorded state directly from each model's equations:

\begin{itemize}
    \setlength{\itemsep}{2pt}
    \item \textbf{Polynormer-r.} Each global layer records the key-value moment \(K^\top V\) and the key sum \(K^\top\mathbf 1\). Polynomial gates, nodewise queries, local layers, and the readout remain active.

    \item \textbf{SGFormer.} The source implementation globally normalizes its queries and keys. We therefore record the query norm, the normalized key-value moment, and the normalized key sum. The GCN branch, current nodewise queries and values, and the self term in the SGFormer reader remain active.

    \item \textbf{NodeFormer.} For each layer and head, we record
    \[
        S=\phi(K)^\top V,
        \qquad
        \zeta=\phi(K)^\top\mathbf 1.
    \]
    The random-feature bases are fixed for the run and remain part of the reader. Gumbel sampling, relational bias, and the edge loss are disabled in this audit so that the tested path is deterministic.

    \item \textbf{DIFFormer-s.} For the simple-kernel variant, we record the query norm, normalized key-value moment, value sum, and key sum. The local graph branch, residual paths, and nodewise query computation remain active. This intervention does not apply to the non-factorized DIFFormer-a attention.
\end{itemize}

This gives \(4\times4=16\) independently trained model-graph combinations. For every pair, the maximum absolute difference between the original and replayed logits was zero up to numerical stability treshold, confirming what mathematics shows from the formulation.

\subsection{Learning the Shared Memory Directly}
\label{sec:exp-learned-memory}

Exact replay is a post-training equality. It does not show that the same state can be learned directly. We test this by comparing two configurations on Cora, Amazon-Photo, PubMed, and Actor:

\begin{enumerate}
    \setlength{\itemsep}{2pt}
    \item the \emph{original model}, the original model as implemented in the code base provided by the authors;
    \item a \emph{learned-memory model}, which replaces the writer by independently initialized trainable states of the same reader-compatible shapes.
\end{enumerate}

We use SGFormer, NodeFormer, DIFFormer-s, and Polynormer-r. The learned-memory model retains every node-side reader and every local, residual, or sparse graph path. Only the key-side writer map is removed from optimization. Its output is replaced by one trainable state per layer. These states are initialized from an independent random distribution.

The original and learned-memory models use identical initial values for every common tensor and identical random-number-generator states. The learned states use zero weight decay and a separately selected learning rate. For SGFormer and DIFFormer-s, the active parameter counts differ by at most \(0.4\%\). For NodeFormer, the number of trainable values in the learned memory depends on the random-feature dimension $m$, while the removed key projection depends on the per-head width $d_v$. We set $m=d_v$, making the replacement memory and the removed key projection exactly equal in parameter count. We also ran the source setting $m=30$ and obtained the same qualitative result. The parameter-matched setting is shown in \cref{fig:learned-memory-bars}.

We hyperparameter-tune to choose state learning rates and state initialization scales using the validation set and on a separate set of random seeds, then evaluate the finalized hyperparameter models on the original five random seeds.

The learned-memory model is within one percentage point of the original model in fourteen of sixteen comparisons. Its mean is higher in eleven comparisons and lower in five. The only two differences larger than one point favor opposite configurations. There is therefore no consistent accuracy advantage from recomputing the shared memory in these runs.

These results do not prove that the training procedures or hypothesis classes are equivalent. They show that directly learned graph-specific states can be competitive across several standard transductive settings.

\subsection{Response to Input Perturbations}
\label{sec:exp-perturbations}

The previous experiment keeps the input fixed. We next construct a task in which the correct shared memory changes in a known way. We use the largest connected components of Cora and Actor. Let $X\in\mathbb{R}^{n\times 8}$ denote the node signals. Each row \(x_i\) is an eight-dimensional Gaussian vector assigned to node \(i\). These are the values that the model must aggregate.

We also select \(64\) fixed nodes as landmarks and construct
$\Phi_G\in\mathbb{R}^{n\times 64}$. The row \(\phi_i^\top\) describes the position of node \(i\) relative to the landmarks. Before normalization, its entries are
\[
(\Phi_G)_{i\ell}
\propto
\exp\left(-\frac{d_G(i,\ell)}{\tau}\right),
\]
where \(d_G(i,\ell)\) is the graph distance from node \(i\) to landmark \(\ell\). We normalize each landmark column and apply one common scale. The model receives the concatenated node input \([\Phi_G,X]\).

The two matrices therefore have different roles. \(X\) contains the node information to be averaged, while \(\Phi_G\) determines which nodes should receive larger weights. In particular, the target for node \(i\) is
\[
y_i
=
\sum_j a_{ij}x_j,
\qquad
a_{ij}
=
\frac{\phi_i^\top\phi_j}
{\sum_k\phi_i^\top\phi_k}.
\]
Nodes with similar landmark-distance representations receive larger weights. In matrix form,
\begin{equation}
    Y(G,X)
    =
    D_G^{-1}\Phi_G\Phi_G^\top X,
    \qquad
    D_G
    =
    \operatorname{diag}
    \left(\Phi_G\Phi_G^\top\mathbf 1\right).
    \label{eq:perturbation-target}
\end{equation}

This target has the exact compact shared state
\[
S_G=\Phi_G^\top X,
\qquad
\zeta_G=\Phi_G^\top\mathbf 1.
\]
Using its own structural representation, node \(i\) reads
\[
y_i
=
\frac{\phi_i^\top S_G}
{\phi_i^\top\zeta_G}.
\]
The output is an eight-dimensional vector and generally differs across nodes. The median target-kernel mass beyond two hops is \(94.5\%\) on Cora and \(97.9\%\) on Actor.

Under feature perturbations, \(X\) changes while \(G\) and \(\Phi_G\) remain fixed. Under graph perturbations, the landmark nodes remain fixed, but the graph distances and \(\Phi_G\) are recomputed.

We train Polynormer-r, SGFormer, a deterministic fixed-random-feature NodeFormer, and DIFFormer-s. We compare two executions of the same checkpoint. The \emph{current writer} recomputes its state from the perturbed input. The \emph{recorder state} reuses the state recorded on the clean input. Both executions receive the current node features, current graph, current queries, and current local paths. They differ only in the compact writer state.

We also include an aligned writer-only control. Its queries and keys are fixed to \(\Phi_G\), and it learns only an \(8\times8\) value map and an output bias. This control checks that the task and optimizer can expose a useful response to input changes.

Feature perturbations use
\begin{equation}
    X_\sigma
    =
    \sqrt{1-\sigma^2}\,X_0+\sigma E,
    \label{eq:feature-perturbation}
\end{equation}
where \(E\) is an independent standard Gaussian matrix. Thus every marginal remains standard Gaussian, and \(\sigma=1\) replaces all features. Topology perturbations use connected degree-preserving double-edge swaps. We recompute both \(\Phi_G\) and the target after each change. The horizontal axis reports the realized fraction of changed edges.

We study two training regimes. The first trains one model on one clean graph and feature matrix. The second trains one shared checkpoint on \(16\) inputs: eight Cora inputs and eight Actor inputs. Each graph uses
\[
\begin{split}
(\sigma,r)\in\{&
(0,0),(.05,0),(.1,0),(0,.05),(0,.1),\\
&(.05,.05),(.1,.05),(.05,.1)\},
\end{split}
\]
where \(r\) is the requested rewiring severity. Training and evaluation perturbations use disjoint random seeds.

Each result uses three optimization seeds and three perturbation draws per nonzero shift. We average perturbation draws within each optimization seed and report test-normalized RMSE.

Results are shown in \cref{fig:input-perturbations}.

All four source-model configurations fit the clean static task better than a train-mean predictor. Under feature replacement, however, their current-writer and recorder-state curves nearly overlap. Their response slopes are close to zero. The models learned useful predictors for the clean input, but did not learn the intended feature-to-memory rule.

Topology changes give mixed results. Recomputing is nearly neutral for SGFormer and Polynormer-r. The recorder state is better for NodeFormer. DIFFormer-s benefits from recomputation on Cora but becomes unstable under strong Actor rewiring. The aligned control gives the expected opposite result: recomputation remains nearly exact, while its recorder state fails as the input moves away from the clean graph.

Training on \(16\) inputs does not produce a general improvement. SGFormer improves consistently across both perturbation types and both graphs, but its current-writer and recorder-state curves still overlap closely. The improvement therefore comes from the complete model and cannot be attributed to the covered writer alone. The other architectures do not improve consistently. These results show that fitting one static input does not identify whether a useful input-conditioned writer has been learned.

\subsection{A Static Target With Known Long-Range Dependence}
\label{app:static-long-range}

We reuse the task from \cref{sec:exp-perturbations}, but keep one input fixed
throughout training and evaluation. We use the largest connected components of
Cora, Actor, Amazon-Photo, and PubMed. Each node receives an
eight-dimensional Gaussian signal, $X\in\mathbb{R}^{n\times 8}$. We choose
$64$ fixed landmarks and construct $\Phi_G\in\mathbb{R}^{n\times 64}$ from
graph distance. Before column normalization,
\[
    (\Phi_G)_{i\ell}
    \propto
    \exp\!\left(-\frac{d_G(i,\ell)}{\tau}\right),
    \qquad \tau=0.75.
\]
The model input is $[\Phi_G,X]$. The target is
\begin{equation}
    Y(G,X)
    =D_G^{-1}\Phi_G\Phi_G^\top X,
    \qquad
    D_G=\operatorname{diag}
    \left(\Phi_G\Phi_G^\top\mathbf 1\right).
    \label{eq:static-long-range-target}
\end{equation}
It has the exact compact state
\[
    S_G=\Phi_G^\top X,
    \qquad
    \zeta_G=\Phi_G^\top\mathbf 1.
\]
Node $i$ reads $y_i=(\phi_i^\top S_G)/(\phi_i^\top\zeta_G)$. The output is
node-specific, while most target mass lies outside a two-hop neighborhood on
three of the four graphs.

\begin{table}[H]
\centering
\caption{\textbf{Graph structures used by the static long-range task.}
Diameter is the exact unweighted diameter of the largest connected component.
The last column is the median target-kernel mass outside two hops.}
\label{tab:static-long-range-graphs}
\small
\setlength{\tabcolsep}{7pt}
\begin{tabular}{@{}lrrr@{}}
\toprule
Graph & Nodes & Diameter & Mass beyond two hops \\
\midrule
Cora          & $2{,}485$  & $19$ & $94.5\%$ \\
Actor         & $7{,}600$  & $12$ & $97.9\%$ \\
Amazon-Photo  & $7{,}487$  & $11$ & $77.6\%$ \\
PubMed        & $19{,}717$ & $18$ & $99.8\%$ \\
\bottomrule
\end{tabular}
\end{table}

We compare the original SGFormer, DIFFormer-s, NodeFormer, and Polynormer-r with two
counterparts. The first removes the global writer and optimizes one
reader-compatible shared state at each layer. It retains all readers, local
paths, and residual paths. The second is a GNN with exactly two
message-passing layers (GCN) and a nodewise MLP head.

We also report two simple baselines. The train-mean baseline predicts the mean
training target at every node. The nearest-training-node baseline copies the
target of the closest training node in graph distance. It uses no validation
or test targets. We use a deterministic $60/20/20$ split stratified by each
node's nearest landmark. We normalize targets using training-node statistics. \Cref{fig:static-long-range} reports normalized
test RMSE and $95\%$ confidence intervals over five disjoint final seeds.

Directly optimized memory is in the same range as the original writer and has
lower mean RMSE in eight of the sixteen model-graph comparisons. There is no
consistent advantage for either parameterization. The two-hop GNN obtains
$R^2$ values of $.911$, $.873$, $.512$, and $.934$ on Cora, Actor,
Amazon-Photo, and PubMed. It improves on nearest-training-node prediction on
three graphs, but not on Amazon-Photo. Thus a local GNN can learn a strong
graph-specific predictor even when most target dependence lies beyond its
forward receptive field. This does not show that it learns the dynamic global
rule. It only concerns one fixed input.

\paragraph{A local model can learn the fixed global information.}
This task gives a simple example in which a local model can directly optimize
the required global information. When the input changes, the correct memory
\(S_G=\Phi_G^\top X\) must be recomputed. On one fixed input, however, the same
matrix can be learned as ordinary shared weights.

Set \(Q_G=D_G^{-1}\Phi_G\in\mathbb R^{n\times 64}\). The target can then be
written as \(Y=Q_GS_G\), where \(S_G\in\mathbb R^{64\times 8}\). Consider the
zero-hop local model \(\widehat Y(W)=Q_GW\). Here,
\(W\in\mathbb R^{64\times 8}\) is a shared nodewise weight matrix.

\begin{proposition}[Learning the fixed landmark memory]
\label{prop:landmark-memory-learning}
Let \(L\) denote the training nodes. Optimize \(W\) using
\(\mathcal L(W)=\frac12\|Q_{G,L}W-Y_L\|_F^2\). If \(Q_{G,L}\) has full column
rank and \(0<\eta<2/\|Q_{G,L}\|_2^2\), gradient descent converges from every
initialization to \(W_\infty=S_G\). It therefore recovers the target on every
node.
\end{proposition}

\begin{proof}
Since \(Y_L=Q_{G,L}S_G\), the gradient-descent error satisfies
\[
    W_{t+1}-S_G
    =
    \left(I-\eta Q_{G,L}^\top Q_{G,L}\right)(W_t-S_G).
\]
Full column rank makes \(Q_{G,L}^\top Q_{G,L}\) positive definite. Under the
stated step size, all error components converge to zero.
\end{proof}

This result uses the precomputed structural map $Q_G$. It proves that the fixed memory effect is directly learnable, but it does not claim that the two-hop GNN follows the same optimization path.

The numerical rank of \(Q_{G,L}\) is \(64\) on all four graphs. Its condition
numbers are \(3.40\), \(9.41\), \(22.79\), and \(3.28\) on Cora, Actor,
Amazon-Photo, and PubMed. Thus the proposition applies to every input used in
this experiment.

Target standardization does not change the result. Since
\(Q_G\zeta_G=\mathbf 1\), a standardized target can be written as
\(\widetilde Y=Q_G(S_G-\zeta_G\mu^\top)\Sigma^{-1}\), where \(\mu\) and
\(\Sigma\) are the training-target mean and scale. If \(Q_{G,L}\) is rank deficient, Proposition~\ref{prop:affine-state} gives the
more general condition
\(\ker(Q_{G,L})\subseteq\ker(Q_{G,U})\). Every memory direction used by
held-out nodes must also be visible to the training nodes.

\subsection{City-Networks Benchmarks}
\label{app:city-memory}

City-Networks provides a long-range node-classification task on large road graphs \citep{liang2026citynetworks}. The ten classes are obtained from weighted node eccentricity within 16 hops. We use the provided features, labels, and fixed $10/10/80$ train, validation, and test split. We report results on Paris, Shanghai, and Los Angeles.

We compare each original model with a directly optimized shared-memory counterpart. The counterpart removes the global writer and replaces its output at each layer with a same-shaped trainable state. The state is initialized independently and is not initialized from replay. The original readers and all local, relational, and residual paths remain unchanged. Matched seeds preserve the initialization of every retained parameter.

For SGFormer, we use the architecture and main numerical settings released with City-Networks. NodeFormer uses the same depth, width, dropout, and training budget as a controlled adaptation. It is not a published City-specific configuration. Both models use our common training and evaluation loop.

Across the six comparisons in \cref{fig:city-memory}, the difference in mean accuracy ranges from $-1.47$ to $+1.56$ percentage points. All displayed marginal confidence intervals overlap. These preliminary results show that the shared memories can be optimized to similar accuracy on these fixed graphs.

\subsection{Hyperparameters and Reproducibility}
\label{sec:exp-hyperparameters}

All experiments use full-batch training. Graphs are made undirected. Cora and PubMed use the PyG public Planetoid split \citep{yang2016planetoid}, Amazon-Photo uses the stored Polynormer split \citep{platonov2023critical}, and Actor uses public split \(0\) \citep{lim2021large}. Unless stated otherwise, optimization uses Adam, cross-entropy for node classification, and gradient clipping at \(5\).

The exact replay uses \(100\) training epochs, learning rate \(5\times10^{-3}\), weight decay \(10^{-5}\), and early stopping after ten validation checks. Dropout is disabled because replay is tested under deterministic evaluation.

\begin{table*}[t]
\centering
\caption{\textbf{Exact replay configurations.} Depth lists local and global layers when both are present. RF is the number of NodeFormer random features.}
\label{tab:replay-hparams}
\small
\setlength{\tabcolsep}{7pt}
\scalebox{0.9}{
\begin{tabular}{@{}lccccc@{}}
\toprule
Model & Hidden width & Depth & Heads & RF & Additional setting \\
\midrule
Polynormer-r & \(64\) & \(1\) local \(+\) \(1\) global & \(1\) & -- & Layer normalization \\
SGFormer & \(64\) & \(4\) GCN \(+\) \(1\) global & \(1\) & -- & Graph weight \(0.8\) \\
NodeFormer & \(32\) & \(2\) kernel layers & \(4\) & \(30\) & Fixed RF basis, no Gumbel or edge loss \\
DIFFormer-s & \(64\) & \(2\) simple-kernel layers & \(1\) & -- & Local graph branch retained \\
\bottomrule
\end{tabular}
}
\end{table*}

\Cref{tab:static-hparams} gives the configurations used in the static learned-memory comparison. For SGFormer, depth, dropout, and weight decay are written as local/global, GCN/Transformer, and base/writer values. For NodeFormer, the \(H/\mathrm{RF}/r_b\) column gives heads, random features, and relational-bias order. The NodeFormer rows show the exact-count geometry used in \cref{fig:learned-memory-bars}. Learned states always use zero weight decay.

\begin{table*}[t]
\centering
\caption{\textbf{Original and learned-memory hyperparameters.}
\(E/P\) gives the maximum number of epochs and early-stopping patience.
For Polynormer-r, layers give local \(+\) global layers, dropout gives
input/local/global dropout, and \(E/P\) gives local-stage epochs \(+\)
global-stage maximum/patience. Its \(d\) is the width per head. The final
column gives the learned-state learning rate and initialization scale.}
\label{tab:static-hparams}
\scriptsize
\setlength{\tabcolsep}{3.6pt}
\resizebox{\textwidth}{!}{\begin{tabular}{@{}llcccccccc@{}}
\toprule
Model & Graph & \(d\) & Layers & \(H/\mathrm{RF}/r_b\) & Dropout & LR & Weight decay & \(E/P\) & State LR/init. \\
\midrule
SGFormer & Cora          & 64 & \(4+1\) & \(1/--/--\) & \(.5/.2\) & \(10^{-2}\) & \(5{\times}10^{-4}/10^{-3}\) & \(1000/200\) & \(.01/.10\) \\
SGFormer & Amazon-Photo  & 64 & \(6+1\) & \(1/--/--\) & \(.5/.3\) & \(10^{-3}\) & \(5{\times}10^{-5}/5{\times}10^{-5}\) & \(1000/120\) & \(.01/.10\) \\
SGFormer & PubMed        & 64 & \(4+1\) & \(1/--/--\) & \(.5/.3\) & \(5{\times}10^{-3}\) & \(5{\times}10^{-4}/10^{-2}\) & \(1000/200\) & \(.10/.10\) \\
SGFormer & Actor         & 64 & \(2+1\) & \(1/--/--\) & \(.2/.2\) & \(10^{-3}\) & \(5{\times}10^{-4}/10^{-2}\) & \(1000/200\) & \(.01/.10\) \\
\midrule
NodeFormer & Cora         & 64 & 2 & \(4/64/3\) & \(.5\) & \(5{\times}10^{-3}\) & \(5{\times}10^{-4}\) & \(1000/200\) & \(.001/.01\) \\
NodeFormer & Amazon-Photo & 64 & 2 & \(4/64/0\) & \(.3\) & \(10^{-3}\) & \(5{\times}10^{-5}\) & \(1000/120\) & \(.001/.01\) \\
NodeFormer & PubMed       & 32 & 2 & \(4/32/2\) & \(.3\) & \(10^{-3}\) & \(5{\times}10^{-4}\) & \(1000/200\) & \(.001/.01\) \\
NodeFormer & Actor        & 32 & 2 & \(1/32/2\) & \(0\) & \(10^{-4}\) & \(5{\times}10^{-2}\) & \(1000/200\) & \(.001/.01\) \\
\midrule
DIFFormer-s & Cora         & 64 & 8 & \(1/--/--\) & \(.2\) & \(10^{-3}\) & \(10^{-2}\) & \(500/100\) & \(1.0/.01\) \\
DIFFormer-s & Amazon-Photo & 64 & 4 & \(1/--/--\) & \(.2\) & \(10^{-3}\) & \(5{\times}10^{-5}\) & \(1000/120\) & \(.01/.01\) \\
DIFFormer-s & PubMed       & 64 & 4 & \(1/--/--\) & \(.2\) & \(5{\times}10^{-4}\) & \(10^{-3}\) & \(500/100\) & \(.01/.01\) \\
DIFFormer-s & Actor        & 64 & 2 & \(1/--/--\) & \(.5\) & \(10^{-3}\) & \(10^{-1}\) & \(300/80\) & \(.10/.01\) \\
\midrule
Polynormer-r & Cora         & 64 & \(3+2\) & \(1/--/--\) & \(.15/.5/.5\) & \(10^{-2}\) & \(5{\times}10^{-4}\) & \(100+500/60\) & \(.01/.02\) \\
Polynormer-r & Amazon-Photo & 64 & \(7+2\) & \(8/--/--\) & \(.2/.7/.5\) & \(10^{-3}\) & \(5{\times}10^{-5}\) & \(200+1000/120\) & \(.01/.02\) \\
Polynormer-r & PubMed       & 64 & \(3+2\) & \(1/--/--\) & \(.15/.5/.5\) & \(5{\times}10^{-3}\) & \(5{\times}10^{-4}\) & \(100+500/60\) & \(.05/.02\) \\
Polynormer-r & Actor        & 64 & \(2+1\) & \(1/--/--\) & \(.15/.5/.5\) & \(10^{-3}\) & \(0\) & \(100+500/60\) & \(.05/.05\) \\
\bottomrule
\end{tabular}}
\end{table*}

SGFormer uses graph weight \(0.8\). Amazon-Photo enables batch normalization, residuals, and the hidden activation; the other SGFormer configurations disable them. NodeFormer uses \(\tau=0.25\), batch normalization, and residual connections. Gumbel sampling and edge loss are disabled in both paired configurations. DIFFormer-s uses \(\alpha=0.5\), batch normalization, residual connections, and no learned layer weighting. Its graph branch is active except on Actor.

The perturbation-task settings are listed in \cref{tab:shift-task-hparams}. The graph, landmark set, clean signal, split, and perturbation draws are cached and shared across model configurations.

\begin{table*}[t]
\centering
\caption{\textbf{Perturbation-task hyperparameters.}}
\label{tab:shift-task-hparams}
\small
\setlength{\tabcolsep}{6pt}
\begin{tabular}{@{}lp{0.76\textwidth}@{}}
\toprule
Parameter & Value \\
\midrule
Graphs & Largest connected components of Cora and Actor, with \(2485\) and \(7600\) nodes \\
Split & Nearest-landmark-stratified \(60/20/20\); train/validation/test counts are \(1467/480/538\) and \(4530/1526/1544\) \\
Structural map & \(64\) farthest-point landmarks, exponential scale \(\tau=0.75\) \\
Node signal & Eight-dimensional standard Gaussian \\
Feature shifts & \(\sigma\in\{0,.1,.2,.4,.6,.8,1\}\) in \cref{eq:feature-perturbation} \\
Topology shifts & Connected degree-preserving swaps with requested severities \(\{0,.01,.02,.05,.1,.2,.4,.8,1\}\); results use the realized changed-edge fraction \\
Perturbation draws & Three per nonzero shift, averaged within each training run \\
Multi-input training & Eight inputs per graph, \(16\) in total \\
Evaluation metric & Test RMSE after normalizing each target dimension using training-node statistics \\
\bottomrule
\end{tabular}
\end{table*}

\begin{table*}[t]
\centering
\caption{\textbf{Perturbation-model hyperparameters.} One-input budget gives maximum epochs and patience in validation checks. Multi-input models use \(1600\) steps, validation every \(80\) steps, and patience \(10\) checks. All source-model rows use dropout \(0\).}
\label{tab:shift-model-hparams}
\small
\setlength{\tabcolsep}{5.5pt}
\begin{tabular}{@{}lccccc@{}}
\toprule
Model & Width & Depth & Heads/RF & LR/weight decay & One-input budget \\
\midrule
Polynormer-r & 64 & \(1\) local \(+\) \(1\) global & \(1/--\) & \(.005/10^{-5}\) & \(1500/50\) \\
SGFormer & 64 & \(4\) GCN \(+\) \(1\) global & \(1/--\) & \(.005/10^{-5}\) & \(300/30\) \\
NodeFormer & 32 & \(2\) kernel layers & \(4/30\) & \(.005/10^{-5}\) & \(1000/50\) \\
DIFFormer-s & 64 & \(2\) simple-kernel layers & \(1/--\) & \(.005/10^{-5}\) & \(1000/30\) \\
Writer-only control & 8 & \(8{\times}8\) value map \(+\) bias & -- & \(.05/0\) & \(3000/50\) \\
\bottomrule
\end{tabular}
\end{table*}

\begin{table}[H]
\centering
\caption{\textbf{Hyperparameters for the City-Networks experiment.}
State learning rates are listed as Paris/Shanghai/Los Angeles.}
\label{tab:city-hparams}
\scriptsize
\setlength{\tabcolsep}{4.5pt}
\begin{tabular}{@{}lll@{}}
\toprule
Setting & SGFormer & NodeFormer \\
\midrule
Hidden width & $64$ & $64$ \\
Layers & $16$ local $+$ $16$ global & $16$ \\
Attention heads & $1$ & $1$ \\
Dropout & $0.2$ & $0.2$ \\
Normalization / residual & on / on & on / on \\
Graph weight & $0.8$ & -- \\
Random features & -- & $64$ \\
Relational-bias order & -- & $2$ \\
Temperature $\tau$ & -- & $0.25$ \\
Gumbel sampling / edge loss & -- & off / off \\
\midrule
Optimizer & Adam & Adam \\
Loss & Cross-entropy & Cross-entropy \\
Retained-parameter learning rate & $5\cdot10^{-4}$ & $5\cdot10^{-4}$ \\
Weight decay, retained / state & $10^{-5}$ / $0$ & $10^{-5}$ / $0$ \\
State learning-rate grid & $\{10^{-3},10^{-2},10^{-1}\}$ & $\{10^{-3},10^{-2},10^{-1}\}$ \\
Selected state learning rate & $.01/.10/.01$ & $.001/.001/.001$ \\
State initialization scale & $0.01$ & $0.01$ \\
Maximum epochs & $20{,}000$ & $20{,}000$ \\
Evaluation interval / patience & $100/3{,}000$ & $100/3{,}000$ \\
Gradient-norm clipping & $5$ & $5$ \\
Active values, original / state & $205{,}130/205{,}146$ & $271{,}530/271{,}530$ \\
\bottomrule
\end{tabular}
\end{table}

\end{document}